%% file: main.tex
\documentclass[sort]{article} % For LaTeX2e

 \usepackage[dblblindworkshop, final]{neurips_2025}
\input{math_commands.tex}

\usepackage{hyperref}
\input{preamble}
\pgfplotsset{compat=1.18}

\newcommand{\mname}{{Di-Bregman}}

\newcommand{\method}{\mname{}}

\usepackage{comment}
\usepackage{wrapfig}

\usepackage{soul}
\definecolor{soft_purple}{RGB}{163,146,237} 
\sethlcolor{soft_purple}
\title{One-step Diffusion Models with Bregman Density Ratio Matching}

\author{Yuanzhi Zhu\thanks{share the same office}$^{\,\,\,1}$
\quad
Eleftherios Tsonis$^{*1}$
\quad
Lucas Degeorge$^{*1,2,3}$
\quad
Vicky Kalogeiton$^{1}$\\
\scalebox{0.88}{$^1$LIX, École Polytechnique, CNRS, IPP %\quad
$^2$LIGM, École Nationale des Ponts et Chaussées, CNRS, IPP
$^3$AMIAD}
}

\begin{document}

\maketitle

\input{sections/0_abstract}

\input{sections/1_intro}

\input{sections/2.5_background}

\input{sections/3_method}
\input{sections/4_experiments}

\input{sections/5_conclusion}

\clearpage

% % %%%%%%%%% Acknowledgements
% % \newpage
% \section*{Acknowledgements} 
% {This work was supported by ANR-22-CE23-0007, ANR-22-CE39-0016, Hi!Paris grant and fellowship, DATAIA Convergence Institute as part of the “Programme d’Investissement d’Avenir” (ANR-17-CONV-0003) operated by Ecole Polytechnique, IP Paris, and was granted access to the IDRIS High-Performance Computing (HPC) resources under the allocation 2024-AD011014300R1 and 2025-AD011015894} made by GENCI and {mesoGIP of IP Paris}. 

%%%%%%%%% REFERENCES

\bibliography{main}
\bibliographystyle{plain}

\clearpage
%%%%%%%% SUPPLEMENT
% \onecolumn  % Switch to one column layout
% \setcounter{section}{0}
% \renewcommand{\thesection}{\Alph{section}}
\appendix
\input{sections/X_suppl}

\end{document}

%% file: math_commands.tex
\usepackage{amsmath,amsfonts,bm}

\def\eqref#1{equation~\ref{#1}}
\def\1{\bm{1}}

\def\dd{{\mathrm{d}}}

\DeclareMathAlphabet{\mathsfit}{\encodingdefault}{\sfdefault}{m}{sl}
\SetMathAlphabet{\mathsfit}{bold}{\encodingdefault}{\sfdefault}{bx}{n}

\newcommand{\E}{\mathbb{E}}

%% file: preamble.tex
\usepackage{lipsum}
\usepackage{moresize}
\usepackage{multirow}
\usepackage[capitalize]{cleveref}
\usepackage{wrapfig}
\usepackage{amsmath}
\usepackage{amssymb}
\usepackage{amsthm}
\usepackage{mathtools}
\usepackage{booktabs}
\usepackage{cases}
\usepackage{algorithm}
\usepackage{algpseudocode}
\usepackage{xcolor}

\usepackage[export]{adjustbox}
\newtheorem{theorem}{Theorem}

\usepackage{overpic}
\usepackage{rotating}
\usepackage{cancel}
\usepackage{subcaption}
\usepackage{pgfplotstable} % Load the pgfplotstable package for handling data
\usepackage{tikz}
\usepackage{tikz-cd}
\usepackage{colortbl}
\usepackage{cancel}
\usepackage{multibib}
\usepackage{afterpage} % Helps control float placement

\makeatletter
\AddToHook{cmd/appendix/before}{\def\cref@section@alias{appendix}}
\makeatother

\crefname{section}{Sec.}{Secs.}
\Crefname{section}{Section}{Sections}
\Crefname{table}{Table}{Tables}
\crefname{table}{Tab.}{Tabs.}
\Crefname{append}{Appendix}{Appendixs}
\crefname{append}{Append.}{Appends.}
\Crefname{subfigure}{Figure}{Figures}
\crefname{subfigure}{Fig.}{Figs.}

%% file: sections/0_abstract.tex
\begin{abstract}

Diffusion and flow models achieve high generative quality but remain computationally expensive due to slow multi-step sampling. Distillation methods accelerate them by training fast student generators, yet most existing objectives lack a unified theoretical foundation. 
In this work, we propose \method{}, a compact framework that formulates diffusion distillation as Bregman divergence-based density-ratio matching. This convex-analytic view connects several existing objectives through a common lens. 
% In practice, we estimate density ratios via a lightweight classifier, avoiding repeated teacher simulation. % this is a bit inaccurate and not necessary in abstract here for me (yz)
Experiments on CIFAR-10 and text-to-image generation demonstrate that \method{} achieves improved one-step FID over reverse-KL distillation and maintains high visual fidelity compared to the teacher model. Our results highlight Bregman density-ratio matching as a practical and theoretically-grounded route toward efficient one-step diffusion generation.

% We present \method{}, a compact Bregman-divergence framework for distilling diffusion teachers into fast one-step student generators by directly matching the density ratio $r=q/p$ to constant 1. 
% Our theory yields a closed-form gradient (Theorem~\ref{thm:bregman-gradient}) that recovers KL/MSE as special cases and prescribes a principled density ratio weighting $h''(r)r$; in practice we estimate $r$ through a classifier, enabling efficient training (and optional adversarial refinement) using real data instead of costly teacher simulation. Empirically, \method{} achieves strong  sample quality with drastically fewer sampling steps.
    
\end{abstract}

%% file: sections/1_intro.tex
\section{Introduction}

Diffusion models have become a cornerstone of generative modeling, attaining state-of-the-art performance across modalities and tasks \cite{rombach2022high, peebles2023scalable,dufour2024don,lu2024mace,wang2025akira,degeorge2025farimagenettexttoimagegeneration,boudier2025dipsy,courant2025pulp,dufour2024don}. Yet their sampling process remains prohibitively slow, often requiring hundreds of network evaluations per sample. This has motivated an active line of research on distillation: training fast student generators that reproduce a pre-trained teacher’s output in one or few steps. Current approaches can be broadly categorized as ODE-based, which learn consistency mappings along the teacher’s probability-flow ODE, and distribution-based, which directly match the generator’s output distribution to that of the teacher or data. 
% The latter family is attractive for its flexibility but lacks a unified and principled formulation.
Recent work has advanced distribution-based distillation beyond simple regression or score-matching losses. Variational Score Distillation (VSD) \cite{wang2024prolificdreamer} and Distribution Matching Distillation (DMD) \cite{yin2024one} define objectives based on reverse-KL between student and teacher models. $f$-distill \cite{xu2025one} reframed these methods through the lens of $f$-divergences. Despite this progress, a general perceptive that explains these objectives in a simple mathematical form remains missing.

We introduce \method{}, a general framework that formulates diffusion distillation as Bregman divergence-based density-ratio matching. The central insight is that aligning the student distribution $q(x)$ with the teacher $p(x)$ can be viewed as driving the ratio $r(x) = \frac{q(x)}{p(x)}$ toward constant one, under a suitable convex function $h$. This perspective yields a closed-form gradient (Theorem \ref{thm:bregman-gradient}) with weighting $h''(r)r$. Under this formulation, familiar objectives, such as KL- or MSE-based distillation arise as specific choices of $h$. The result is a concise, interpretable expression that connects multiple existing formulations within a single theoretical framework.

Beyond theory, \method{} remains practical. To get the weightning coefficient $h''(r)r$, we estimate density ratios through a simple classifier trained to distinguish student samples from real data, enabling efficient training without repeated teacher simulation and allowing optional adversarial refinement. 
% Preliminary results on unconditional image and text-to-image generation show that this formulation attains improved one-step FID over reverse-KL distillation, and preserving high visual fidelity compared to the multi-step teacher.
Preliminary results on both unconditional image and text-to-image generation demonstrate that our approach attains improved one-step FID than reverse-KL distillation and maintains visual fidelity comparable to the multi-step teacher models.

In summary, our contributions are:

\begin{itemize}
    \item We introduce a unified formulation of diffusion distillation based on Bregman density-ratio matching, which yields a closed-form gradient interpretation, 
    \item We propose a practical classifier-based training procedure that effectively instantiates this formulation and validate it on early benchmarks.
\end{itemize}

%% file: sections/2.5_background.tex
\section{Preliminaries}
\label{sec:background}

\subsection{Variational Score Distillation}
Variational Score Distillation (VSD) \citep{wang2024prolificdreamer} was introduced to mitigate mode-seeking and over-saturation  
% \vic{I would add a sentence (or footnote) to explain these two} 
\footnote{The SDS objective tends to produce solutions corresponding to the mode of the averaged likelihood, leading to mode-seeking behavior. Moreover, a high Classifier Free Guidance (CFG) scale can cause over-saturated and over-smoothed generation results.} issues observed when using Score Distillation Sampling (SDS) for 3D asset generation \citep{poole2022dreamfusion}. Importantly, the VSD objective is defined on the \emph{final} samples produced by a generator, rather than on intermediate sampler states. 
This final-sample focus naturally motivates efforts to distil powerful multi-step pre-trained model into compact few-step or one-step generators via VSD-style objectives; several recent works have followed this route \citep{luo2023diff,yin2024one,nguyen2024swiftbrush}.

Concretely, VSD can be viewed as minimizing a time-averaged divergence between the noisy marginal produced by the student generator and the corresponding noisy marginal of a pretrained reference model. Writing $q_{t}$ and $p_{t}$ for the generator and reference noisy marginals at time $t$, respectively, the gradient of a typical score-distillation loss admits the following approximation:
\begin{equation}\label{eq:kl-grad}
% \small
\begin{aligned}
\nabla_\theta\mathcal{L}_\text{VSD} &=
    \mathbb{E}_t\left[\nabla_\theta \text{KL}(q_{t} \; \| \; p_{t}) \right]  \approx -\mathbb{E}_{t,\epsilon}\left[w(t)\big(s_\phi({x}_t, t) - s_\psi({x}_t, t)\big) \frac{\mathrm{d}G_\theta(\epsilon)}{\mathrm{d}\theta} \right],
\end{aligned}
\end{equation}
where $w(t)$ is a scalar weighting function over timesteps, and the noisy state ${x}_t$ is obtained by applying the forward diffusion kernel 
% $p_{t|0}$ 
at time $t$ to the generator output $G_\theta(\epsilon)$. 
$s_\phi(\cdot,t)$ and $s_\psi(\cdot,t)$ denote the pre-trained score function on reference data and the auxiliary score function on the student-generated data evaluated at timestep $t$, respectively. 
Intuitively, the score difference $s_\phi-s_\psi$ provides a learning signal that pushes the student’s generated noisy marginals toward those of the pre-trained teacher model, and backpropagating through $G_\theta$  to update the generator parameters $\theta$.

\subsection{Bregman Divergence for Density Ratio Matching}
\label{sec:Breg_DRM}
Given two probability distributions $p^*(x)$ and $q^*(x)$, the goal of \emph{density ratio matching} is to learn a ratio model $r_\theta(x)$ that approximates the true density ratio $r^*(x) \coloneqq \frac{q^*(x)}{p^*(x)}$ based on i.i.d. samples drawn from both distributions.  

The \emph{Bregman divergence} provides a flexible and theoretically grounded measure for comparing functions such as density ratios. It generalizes the notion of squared Euclidean distance to a broad class of divergences that share similar geometric and convexity properties \cite{sugiyama2012density, banerjee2005clustering}.
Let $h$ be a differentiable and strictly convex function. The Bregman divergence associated with $h$ between two functions $r$ and $r^*$ is defined as \cite{sugiyama2012density,kim2025preference}:
\begin{equation}
\label{eq:breg_general}
    D_h(r \| r^*) = \int p(x) \Big[ h(r(x)) - h(r^*(x)) - h'(r^*(x)) (r(x) - r^*(x)) \Big] \, \dd x.
\end{equation}

This divergence is positive-definite, which means it is always non-negative and equals zero if and only if $r(x) = r^*(x)$ almost everywhere with respect to $p(x)$, which is the density implicitly defined in $r(x)$.
Many well-known divergences arise as special cases of the Bregman divergence for specific choices of the convex function $h$. For instance, the \emph{squared loss} corresponds to $h(r) = \tfrac{1}{2}r^2$, leading to least-squares density ratio estimation \cite{sugiyama2008direct} and the \emph{KL divergence} corresponds to $h(r) = r \log r - r$.
% The \emph{Itakura–Saito divergence} corresponds to $h(r) = -\log r + r - 1$.
More instances can be found in \cref{tab:bregman_instances}.
This unifying framework allows density ratio estimation to be interpreted as minimizing a Bregman divergence under different convex function $h$, providing a general connection between statistical divergences and convex analysis \cite{bregman1967relaxation,banerjee2005clustering}.

%% file: sections/3_method.tex
\section{Method}
\label{sec:method}

In this section, we introduce a general distillation framework, termed \method{}, which is derived from the Bregman divergence for density ratio matching formulation in \cref{sec:Breg_DRM}. 
The core idea is to align the student distribution $q(x)$, induced by a one-step generator $G_\theta$, with the teacher distribution $p(x)$.
Since the student distribution $q(x)$ is implicitly defined by the generator through the push-forward measure of the prior, i.e., $x = G_\theta(\epsilon)$ with $\epsilon \sim \mathcal{N}(0,I)$, the distribution $q(x)$ and its density ratio depend on the generator parameters $\theta$.  
Let $r(x) = \frac{q(x)}{p(x)}$ denote the density ratio between the student and teacher distributions. 
Perfect alignment hence corresponds to $r(x) = 1$ for all $x$, which motivates minimizing a divergence between $r(x)$ and the target ratio $1$ in \cref{eq:breg_general}:
% We adopt the Bregman divergence as a general measure for this alignment objective. For a differentiable and strictly convex generator function $h$, the distillation objective is defined as:
\begin{equation}\label{eq:berg_1}
    D_h(r \| 1)
    = \int p(x) \Big[ h(r(x)) - h(1) - h'(1)(r(x) - 1) \Big] \, \dd x.
\end{equation}
Minimizing this divergence with respect to $\theta$ encourages the student generator $G_\theta$ to produce samples whose induced distribution $q(x)$ matches the teacher distribution $p(x)$.

Following prior work on one-step and few-step distillation of diffusion models~\citep{luo2023diff, yin2024one, nguyen2024swiftbrush}, we can derive the analytical form of the gradient of the Bregman divergence in \cref{eq:berg_1}. 
The resulting expression corresponds to a weighted variant of the gradient used in the KL-based objective (\cref{eq:kl-grad}), where the weight is a function of the density ratio $r(x)$, analogous to the formulation in $f$-distill~\citep{xu2025one}.

To further generalize this result as in VSD, we consider the intermediate distributions $p_t$ and $q_t$ obtained via the diffusion forward process. 
This allows the Bregman-based distillation gradient to be evaluated at arbitrary diffusion timesteps.
The following theorem formally characterizes the gradient of the Bregman divergence in this general setting.

% We can further generalize this formulation to  Bregman-based distillation gradient to be evaluated at arbitrary diffusion timesteps, as summarized in the foloowing Theorem:

\begin{theorem}[Gradient of Bregman divergence]\label{thm:bregman-gradient}
Let $p_t$ be a reference (teacher) marginal density at time $t$ and let
$q_t=q_{\theta,t}$ be the marginal induced by the generator $G_\theta$ at time $t$. These intermediate densities are obtained via the forward diffusion process.
Define the intermediate density ratio $
r_t(x)\coloneqq \frac{q_{\theta,t}(x)}{p_t(x)}$.
Assume that $h$ is twice continuously differentiable.
Then the gradient of the Bregman divergence $D_h(r_t\|1)=\mathbb{E}_{p_t}[h(r_t)]-h(1)$
with respect to $\theta$ admits the following form:
\begin{equation}\label{eq:breg-loss-theorem}
% \small
\nabla_\theta D_h(r_t\|1)
= -\mathbb{E}_{\epsilon}\Big[\,h''(r_t(x_t))\,r_t(x_t)\,\big(\nabla_{x_t}\log p_t({x_t})-\nabla_x\log q_{\theta,t}({x_t})\big)\,
\nabla_\theta G_\theta(\epsilon)\Big].
% _{{x_t}=G_\theta(\epsilon)}.
\end{equation}
\end{theorem}

\begin{proof}
Starting from the generalized Bregman divergence:
\begin{equation}
    D_h(r_t \| 1)
    = \int p_t(x) \Big[ h(r_t(x)) - h(1) - h'(1)(r_t(x) - 1) \Big] \, \dd x,
\end{equation}
where $r_t(x)=q_{\theta,t}(x)/p_t(x)$. The reference marginal $p_t(x)$ does not depend on $\theta$, hence differentiating under the integral sign gives:
\begin{equation}
    \nabla_\theta D_h(r_t\|1) = \int p_t(x)\,h'(r_t(x))\,\nabla_\theta r_t(x)\,\dd x - h'(1)\int p_t(x)\,\nabla_\theta r_t(x)\,\dd x.
\end{equation}
The second integral vanishes because
\begin{equation}
    \int p_t(x)\,\nabla_\theta r_t(x)\,\dd x = \nabla_\theta\int p_t(x) r_t(x)\,\dd x = \nabla_\theta\int q_{\theta,t}(x)\,\dd x = \nabla_\theta 1 = 0.
\end{equation}
Hence
\begin{equation}\label{eq:proof-step1}
\nabla_\theta D_h(r_t\|1)
= \int p_t(x)\,h'(r_t(x))\,\nabla_\theta r_t(x)\,\dd x = \int h'(r_t(x))\,\nabla_\theta q_{\theta,t}(x)\,\dd x.
\end{equation}
Utilizing the following reparameterization identity:
\[
\nabla_\theta \E_{x\sim q_{\theta,t}}[g(x)] \;=\; \E_{z\sim p(z)}\big[\nabla_\theta g(G_\theta(z))\big].
\]
Set $g(x)=h'(r_t(x))$. Then
\[
\nabla_\theta D_h(r_t\|1)=\nabla_\theta \E_{q_{\theta,t}}[h'(r_t(x))]
= \E_{\epsilon}\!\big[\nabla_x h'(r_t(x))\big]_{x=G_\theta(\epsilon)} \nabla_\theta G_\theta(\epsilon).
\]
Apply the chain rule $\nabla_x h'(r_t)=h''(r_t)\nabla_x r_t$ and $\nabla_x r_t = r_t(\nabla_x\log q_{\theta,t}-\nabla_x\log p_t)$ yields \cref{eq:breg-loss-theorem} stated in the theorem.
\end{proof}

In practice, the density ratio on noisy data, $r_t(x) = \frac{q_t(x)}{p_t(x)}$, can be estimated using a classifier trained to distinguish samples from the student generator $G_\theta$ and those from the teacher model or reference dataset. Under the common assumption that the pre-trained teacher model already captures the data distribution well, it is often both preferable and computationally cheaper to draw real samples directly from the dataset rather than repeatedly sampling from the teacher.
For a discriminator output  $\sigma(l_t(x))$, where $l_t(x)$ denotes the classifier logits at noise level $t$, the optimal output satisfies
$\sigma(l_t^*(x)) = \frac{p_t(x)}{p_t(x) + q_t(x)}$.
This implies that the density ratio can be recovered as $r_t(x) = e^{-l_t(x)}$.
In this framework, the trained classifier not only provides an estimate of the local density ratio but can also be repurposed as a discriminator for adversarially training the student generator.

% In practice, we estimate the density ratio on noisy data $r_t(x)$ with a classifier trained to distinguish samples from the student generator $G_\theta$ and samples from the teacher model. 
% Under the common assumption that the pre-trained teacher model has already captured the training distribution well, it is often preferable and computationally cheaper to draw real samples directly from the dataset instead of repeatedly simulating them from the teacher.
% For discriminator output $\sigma(l_t)$ where $l_t(x)$ is the logits for noisy input at timestep $t$, we have $\sigma(l_t^*)=\frac{p_t}{p_t+q_t}$, hence we have the density ratio $r_t(x)=e^{-l_t}$. 
% In this setting the learned classifier can also be repurposed as the discriminator in adversarial training of the student generator.
% We note, however, that practical use of a classifier-based ratio estimate entails standard considerations such as classifier calibration and support mismatch between \(p\) and \(q\).

Compared to $f$-distill~\citep{xu2025one}, our framework places fewer constraints on the convex function, which yields greater flexibility in choosing divergence families for distillation. 
Recently, Uni-Instruct \citep{wang2025uni} proposes a unifying view that connects integral $f$-divergences \citep{yin2024one,xu2025one} and score-based divergences \citep{zhou2024score,luo2024one}.
However, \method{} is complementary to this line of work: it provides a Bregman-divergence perspective that admits a broader class of function $h$ and recovers many existing objectives as special cases.
Together, these formulations offer a more complete picture of distribution-based diffusion distillation.

%% file: sections/4_experiments.tex
\section{Experiments}
\label{sec:experiments}

\input{figs/quali_images_main}

\begin{comment}

\begin{figure}[!t]
  \centering
    \centering
    \includegraphics[width=\linewidth]{figs/cifar_fid_main.pdf}
    \vspace{-2mm}
    \caption{\textbf{Evolution of one-step FID againt number of iterated images}: Bregman divergence acheives a lower one-step FID.}
    \label{fig:cifar_fid_main}
\end{figure}

\end{comment}

\begin{wrapfigure}{r}{0.48\linewidth} % 'r' for right, 'l' for left
  \vspace{-5mm} % adjust vertical position if needed
  \centering
  \includegraphics[width=\linewidth]{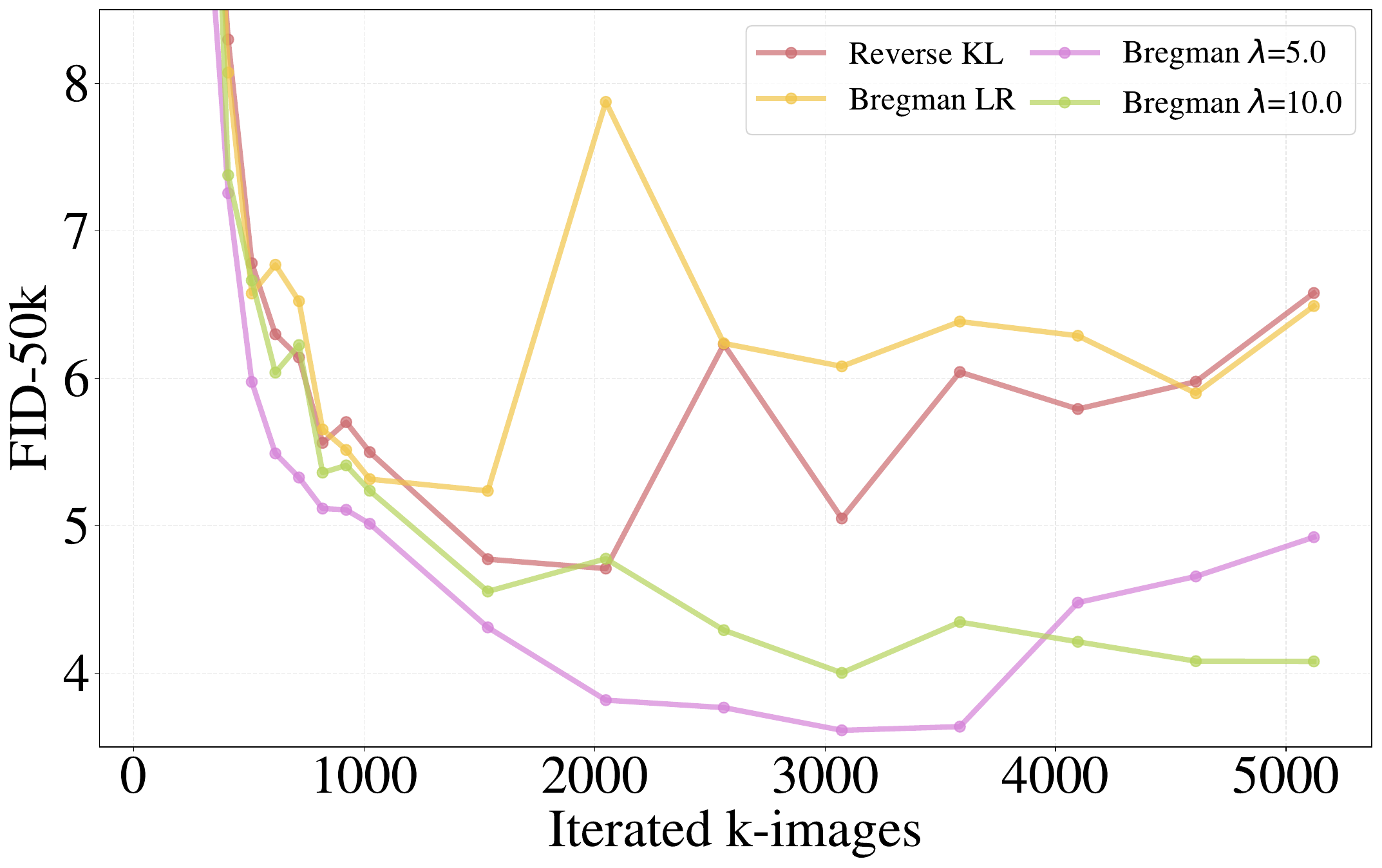}
  \vspace{-2mm}
  \caption{\textbf{Evolution of one-step FID against number of iterated images}: \method{} achieves a lower one-step FID.}
  \label{fig:cifar_fid_main}
\end{wrapfigure}

% To verify the effectiveness of the proposed method, we present quantitative results on CIFAR10 dataset with EDM teacher \cite{karras2022elucidating}, and qualitative results for text-to-image generation task using SD1.5 teacher.
% As shown in \cref{fig:first}, when $\lambda=5$ for SBA type Bregman divergence, we achieve better one-step FID than the baseline reverse KL distillation method.
% In addition, we show visualized one-step generation result of our distilled text-to-image generator.
To evaluate the effectiveness of the proposed method, we conduct experiments on both unconditional image and text-to-image generation tasks. 
Quantitative results are reported on the CIFAR-10 dataset using an EDM teacher \cite{karras2022elucidating}, while qualitative results are presented for text-to-image generation with a Stable Diffusion v1.5~\cite{rombach2022high} teacher.
As shown in \cref{fig:cifar_fid_main}, when applying the SBA-type Bregman divergence with $\lambda = 5$, our method achieves a lower one-step FID compared to the baseline reverse KL distillation approach. In addition, \cref{fig:onestep_images} illustrates representative one-step samples generated by our distilled text-to-image model, demonstrating high visual quality and fidelity to the text prompts.
More experimental details (\ref{sec:expr_setup}), qualitative (\ref{sec:app_quali}), quantitative (\ref{sec:app_quanti}) results and additional ablations (\ref{sec:app_quanti}) are provided in Appendix.

%% file: figs/quali_images_main.tex
\begin{figure*}[t]
    \vspace{-40pt}
    \centering
    \setlength{\tabcolsep}{0.5pt} % Space between images
    \renewcommand{\arraystretch}{0} % Remove extra space between rows
    \begin{tabular}{ccccc}
        \includegraphics[width=0.20\textwidth]{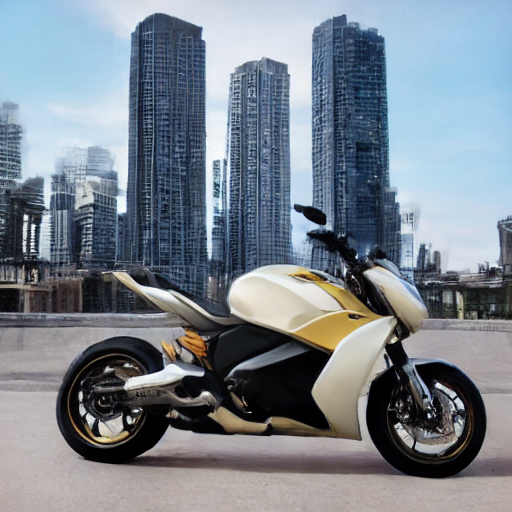} & 
        % \fbox{\rule{0pt}{0.17\textwidth}\rule{0.17\textwidth}{0pt}} &
        
        \includegraphics[width=0.20\textwidth]{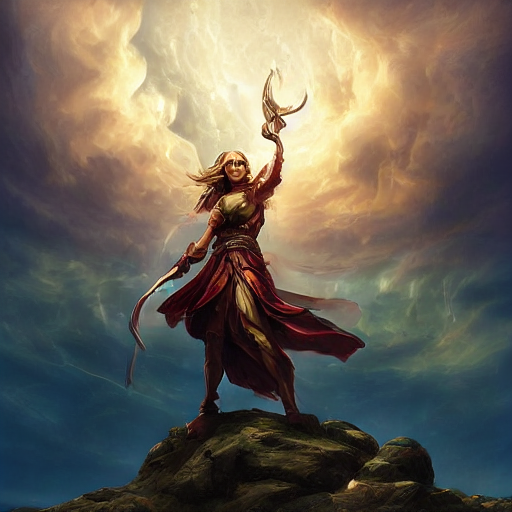} &
        % \fbox{\rule{0pt}{0.17\textwidth}\rule{0.17\textwidth}{0pt}} &
        
        \includegraphics[width=0.20\textwidth]{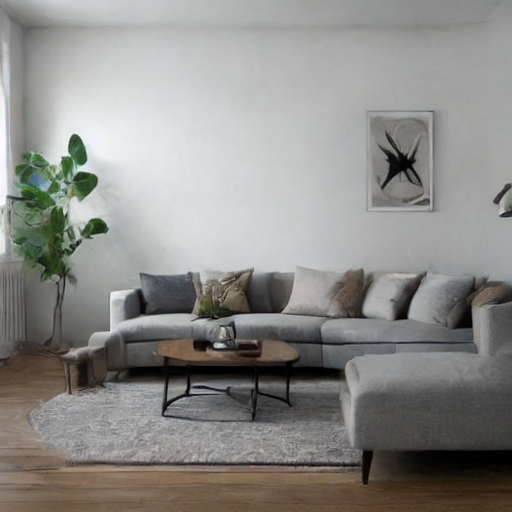} &
        % \fbox{\rule{0pt}{0.17\textwidth}\rule{0.17\textwidth}{0pt}} & 

        \includegraphics[width=0.20\textwidth]{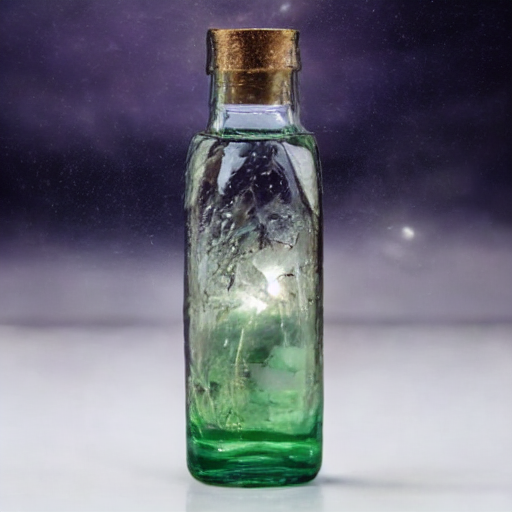} &
        % \fbox{\rule{0pt}{0.17\textwidth}\rule{0.17\textwidth}{0pt}} & 
        
        \includegraphics[width=0.20\textwidth]{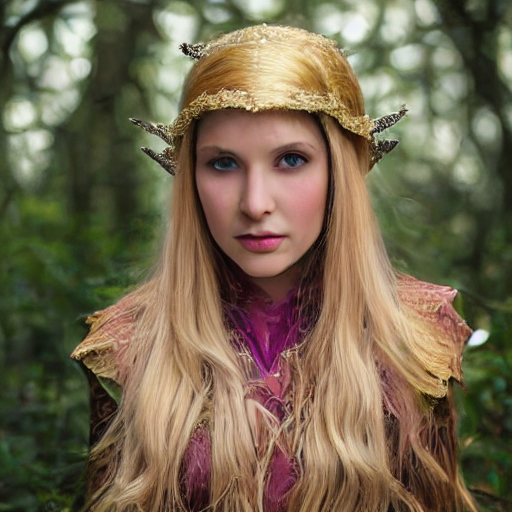} \\ 
        % \fbox{\rule{0pt}{0.17\textwidth}\rule{0.17\textwidth}{0pt}} \\
    \end{tabular}
    \caption{Images generated with only one-step by model trained with \method. More images are shown in Appendix~\ref{sec:app_quali}}
    \label{fig:onestep_images}
\end{figure*}

%% file: sections/5_conclusion.tex
\section{Conclusion}
\label{sec:conclusion}
% In this work
We introduced \method{}, a generalized framework for diffusion model distillation grounded in Bregman divergences.
Empirically, our method improves one-step generation quality on CIFAR-10 and produces competitive visual results in text-to-image synthesis, demonstrating both its theoretical generality and practical effectiveness.
% Future work will explore extending the framework to multi-step generators and other modalities.

%% file: sections/X_suppl.tex
\setcounter{page}{1}

\vspace{0.2cm}
% \noindent\textbf{Limitations and Future Works.}

\section{Limitations and Future Works.}

This work primarily presents preliminary results. 
In future studies, we plan to extend our approach to a wider range of teacher models and conduct comprehensive comparisons with state-of-the-art methods. 
Moreover, while our current experiment only use the classifier in \cref{eq:breg-loss-theorem}, we aim to incorporate adversarial training based on it to further enhance the performance of one-step generation.

%%%% Related work 

\input{sections/2_relatedworks}

%%%% Experimental setup

\section{Experimental Setup}
% \subsection{Experimental Setup}
\label{sec:expr_setup}

\noindent\textbf{Datasets and Pre-trained Teacher Models.}~~
Our experiments to demonstrate the effectiveness of \method{} are performed on the CIFAR-10~\cite{krizhevsky2009learning} 32$\times$32 for unconditional generation and on the LAION~\cite{schuhmann2022laion} and COCO~\cite{lin2014microsoft} datasets for text-to-image generation.
The pre-trained teacher models are adopted from the official checkpoints from previous works, EDM~\cite{karras2022elucidating}, and Stable Diffusion v1.5~\cite{rombach2022high}.

\noindent\textbf{Implementation Details.}~~ 
All experiments are conducted on a single NVIDIA H100 GPU.
For CIFAR-10 ($32\times32$) experiments, we adopt the U-Net architecture of NCSN$++$ \cite{song2020score}.
The implementation is based on the SiDA framework \cite{zhou2024adversarial}, where the discriminator is built upon the auxiliary model encoder, and the mean feature vector is used as the predicted discriminator logits.
For the text-to-image experiment, we use Stable-Diffusion v1.5~\cite{rombach2022high}, a 900M-paramenter U-net-based model, trained on LAION~\cite{schuhmann2022laion} and distilled at $512\times 512$ resolution.
All results presented in the paper are one-step generated using our distilled generator.

\noindent\textbf{Evaluation Metrics.}~~
The metrics we use for quantitative results on CIFAR-10 are Fr\'echet Inception Distance (FID) \cite{heusel2017gans} and Inception Score (IS) \cite{salimans2016improved}.
In our experiments, FID is computed with 50,000 generated samples compared against the training set using Clean-FID \cite{parmar2022aliased}, while IS is calculated from the same generated images based on their Inception features.

%%%% Table on instances

\section{Instances in Bregman divergence}

\begin{table}[!h]
    \centering
    \small
    \label{tab:bregman_instances}
    \renewcommand{\arraystretch}{1.2}
    \setlength{\tabcolsep}{6pt}
    \begin{tabular}{lccc}
        \toprule
        Name & $h(r)$ & $h''(r)r$ & $h''(e^{-l})e^{-l}$ \\
        \midrule
        LR & ${r\log r - (1+r)\log(1+r)}$ & $\frac{1}{1+r}$ & $\sigma(l)$ \\
        KL & $r\log r - r$ & $1$  & $1$ \\
        BE & $-\log r$ & $1/r$   & $e^{l}$ \\
        LS & ${r^2 / 2}$ & $r$ & $e^{-l}$ \\
        % BA & $\frac{r^{1+\lambda}-r}{\lambda}$ & $(\lambda+1)r^{\lambda}$ & $(\lambda+1)e^{-\lambda l}$ \\
        SBA & $\frac{r^{1+\lambda}-r}{\lambda(\lambda+1)}$ & $r^{\lambda}$   & $e^{-\lambda l}$ \\
        \bottomrule
    \end{tabular}
    \setlength{\abovecaptionskip}{6pt}
    \caption{Examples of different $h(r)$ in Bregman divergence and the corresponding $h''(r)r$. The choices of $h(r)$ are from \cite{nielsen2009sided,kim2025preference}.}
    \renewcommand{\arraystretch}{1.0}
\end{table}

\begin{comment}

\begin{wraptable}{r}{0.5\textwidth} % 'r' for right side, adjust width as needed
\vspace{-4mm}
\centering
\small
\caption{Examples of different $h(r)$ in Bregman divergence and the corresponding $h''(r)r$. The choices of $h(r)$ are from \cite{nielsen2009sided,kim2025preference}}
\label{tab:bregman_instances}
\adjustbox{max width=0.53\textwidth}{
\renewcommand{\arraystretch}{1.6}
\begin{tabular}{lccc}
\toprule
Name & $h(r)$ & $h''(r)r$ & $h''(e^{-l})e^{-l}$ \\
\midrule
LR & ${r\log r - (1+r)\log(1+r)}$ & $\frac{1}{1+r}$ & $\sigma(l)$ \\
KL & $r\log r - r$ & $1$  & $1$ \\
BE & $-\log r$ & 1/r   & $e^{l}$ \\
LS & ${r^2 / 2}$ & $r$ & $e^{-l}$ \\
% BA & $\frac{r^{1+\lambda}-r}{\lambda}$ & $(\lambda+1)r^{\lambda}$ & $(\lambda+1)e^{-\lambda l}$ \\
% \midrule
SBA & $\frac{r^{1+\lambda}-r}{\lambda(\lambda+1)}$ & $r^{\lambda}$   & $e^{-\lambda l}$ \\
\bottomrule
\end{tabular}
}
\renewcommand{\arraystretch}{1.0}
\vspace{-3mm}
\end{wraptable}

\end{comment}

\section{Additional qualitative results}
\label{sec:app_quali}

In Fig.~\ref{fig:qualitative_comparison_SD}, we provide additional qualitative comparisons, where our one-step student produces visually coherent and faithful samples, closely matching the teacher output across diverse prompts. Additional uncurated CIFAR-10 samples from our \method{} model are shown in Fig.~\ref{fig:cifar_visual}, demonstrating diverse one-step generation, with an FID of 3.61.

\input{figs/qualitative_images_SD15}

\newpage

\section{Additional quantitative results} 
\label{sec:app_quanti}
In this section, we provide some additional quantitative results from our one-step models.

Figures \ref{fig:cifar_fid_all} and \ref{fig:cifar_is_all} show the evolution of one-step FID and IS, respectively, across different values of the Bregman parameter $\lambda$. We observe that \method consistently improves over the reverse KL baseline for several $\lambda$ configurations. In particular, settings such as $\lambda=3.0$, $\lambda=5.0$, and $\lambda=10.0$ yield the lowest one-step FID (Fig.~\ref{fig:cifar_fid_all}) and the highest one-step IS (Fig.~\ref{fig:cifar_is_all}), confirming the robustness of \method across a range of divergence parameters. Lower $\lambda$ values (e.g., $\lambda\leq 1.0$) tend to perform closer to the baseline, while negative $\lambda=-1.0$ underperforms. These results demonstrate that \method offers consistent improvements in sample quality metrics over the reverse KL distillation method.

\begin{comment}
    
\begin{figure}[!h]
  \centering
  % First figure
  \begin{minipage}[t]{0.48\linewidth}
    \centering
    \includegraphics[width=\linewidth]{figs/cifar_fid_all.pdf}
    \vspace{-2mm}
    \caption{\textbf{Evolution of one-step FID against number of iterated images}. For $\lambda=3.0, \lambda=5.0$ or $\lambda=10.0$ \method{} achieves a lower one-step FID.}
    \label{fig:cifar_fid_all}
  \end{minipage}
  \hfill
  % Second figure
  \begin{minipage}[t]{0.48\linewidth}
    \centering
    \includegraphics[width=\linewidth]{figs/cifar_is_all.pdf}
    \vspace{-2mm}
    \caption{\textbf{Evolution of one-step IS against number of iterated images}. For $\lambda=3.0, \lambda=5.0$ or $\lambda=10.0$ \method{} achieves a higher one-step IS.}
    \label{fig:cifar_is_all}
  \end{minipage}
  \vspace{-3mm}
\end{figure}

\end{comment}

\begin{figure}[!h]
     \centering
     \includegraphics[width=0.6\linewidth]{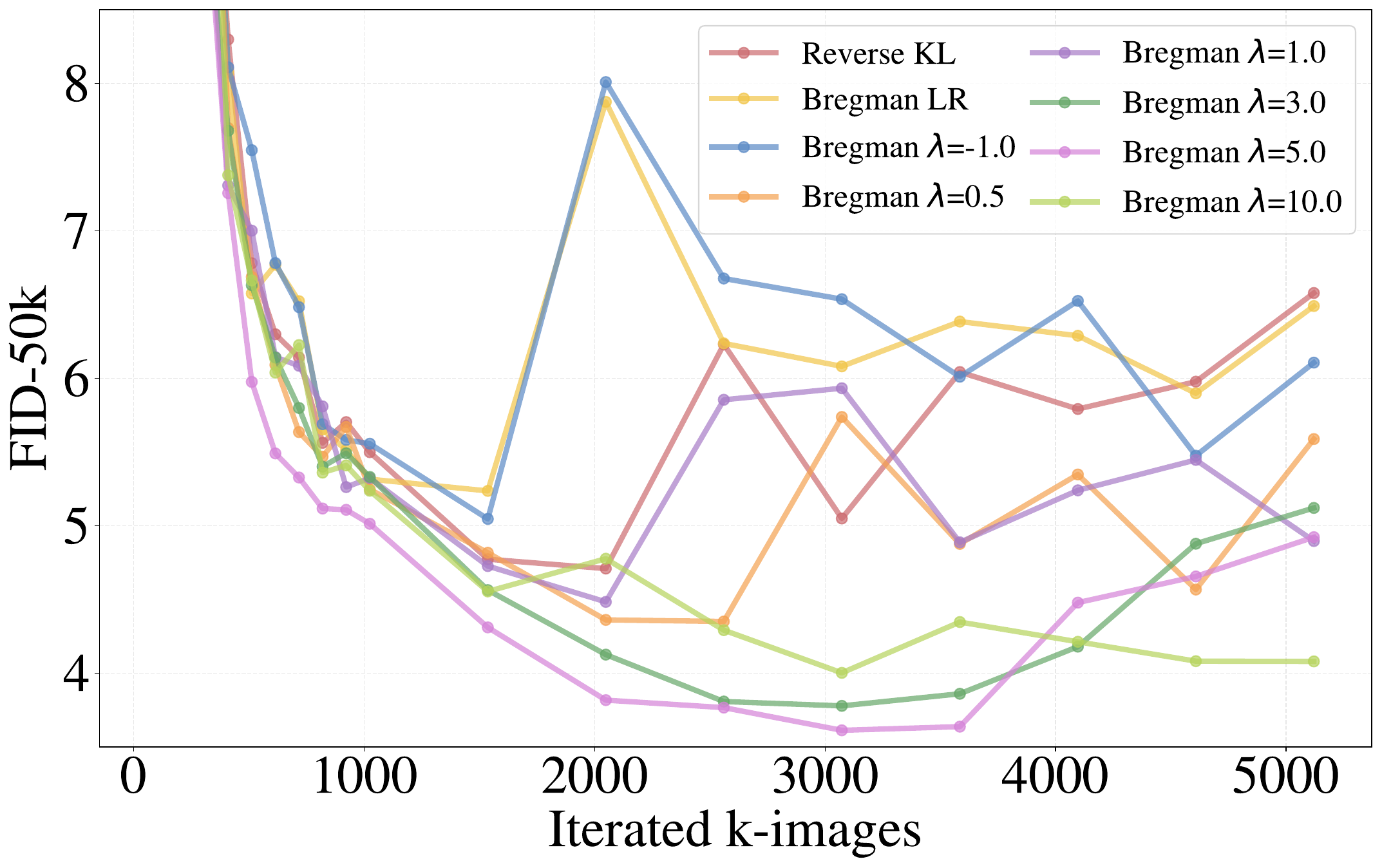}
     \caption{\textbf{Evolution of one-step FID against number of iterated images}. For $\lambda=3.0, \lambda=5.0$ or $\lambda=10.0$ \method{} achieves a lower one-step FID.}
     \label{fig:cifar_fid_all}
\end{figure}

\begin{figure}[!h]
    \centering
    \includegraphics[width=0.6\linewidth]{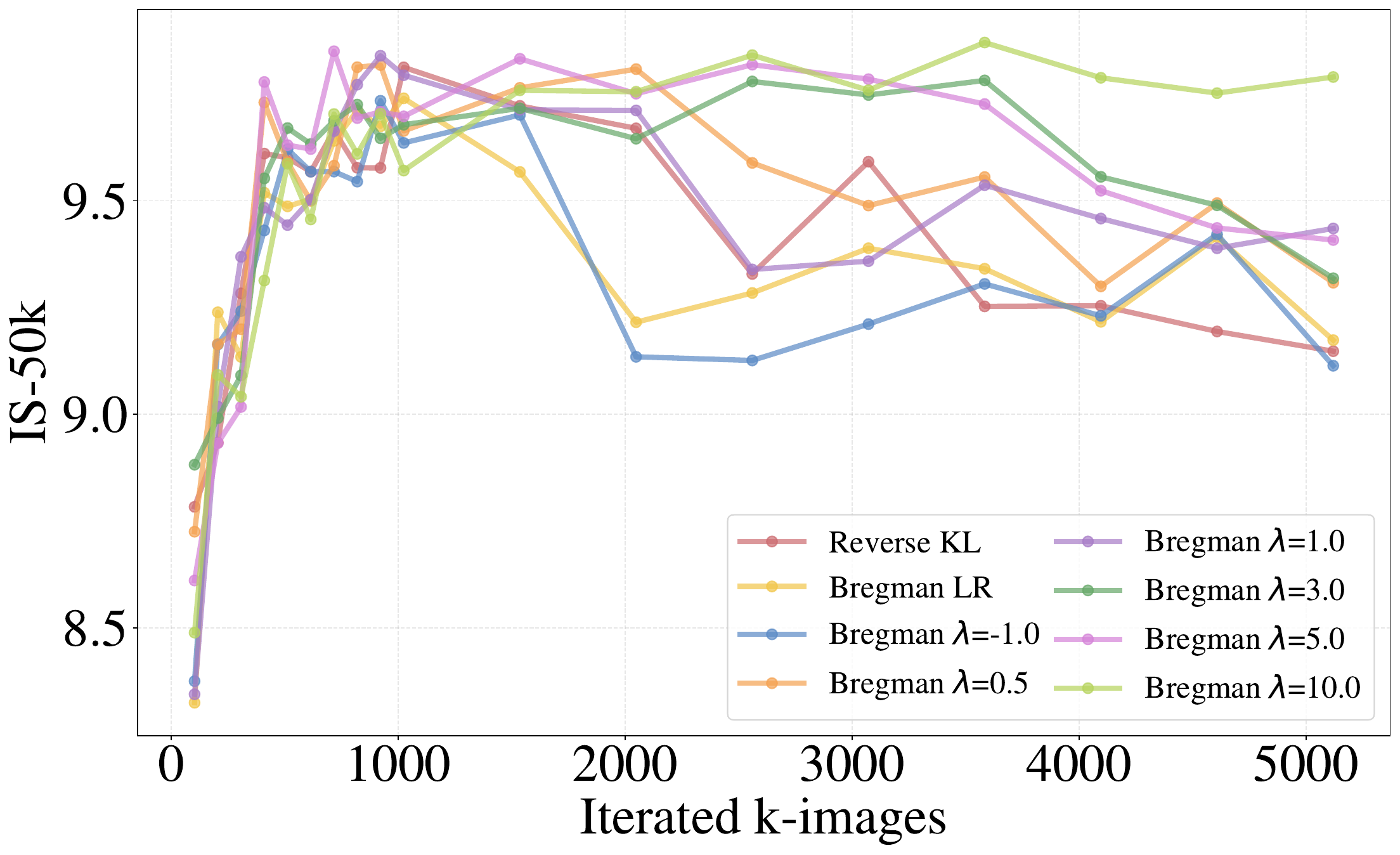}
    \caption{\textbf{Evolution of one-step IS against number of iterated images}. For $\lambda=3.0, \lambda=5.0$ or $\lambda=10.0$ \method{} achieves a higher one-step IS.}
    \label{fig:cifar_is_all}
\end{figure}

\begin{figure}[!ht]
    \centering
    \includegraphics[width=\linewidth]{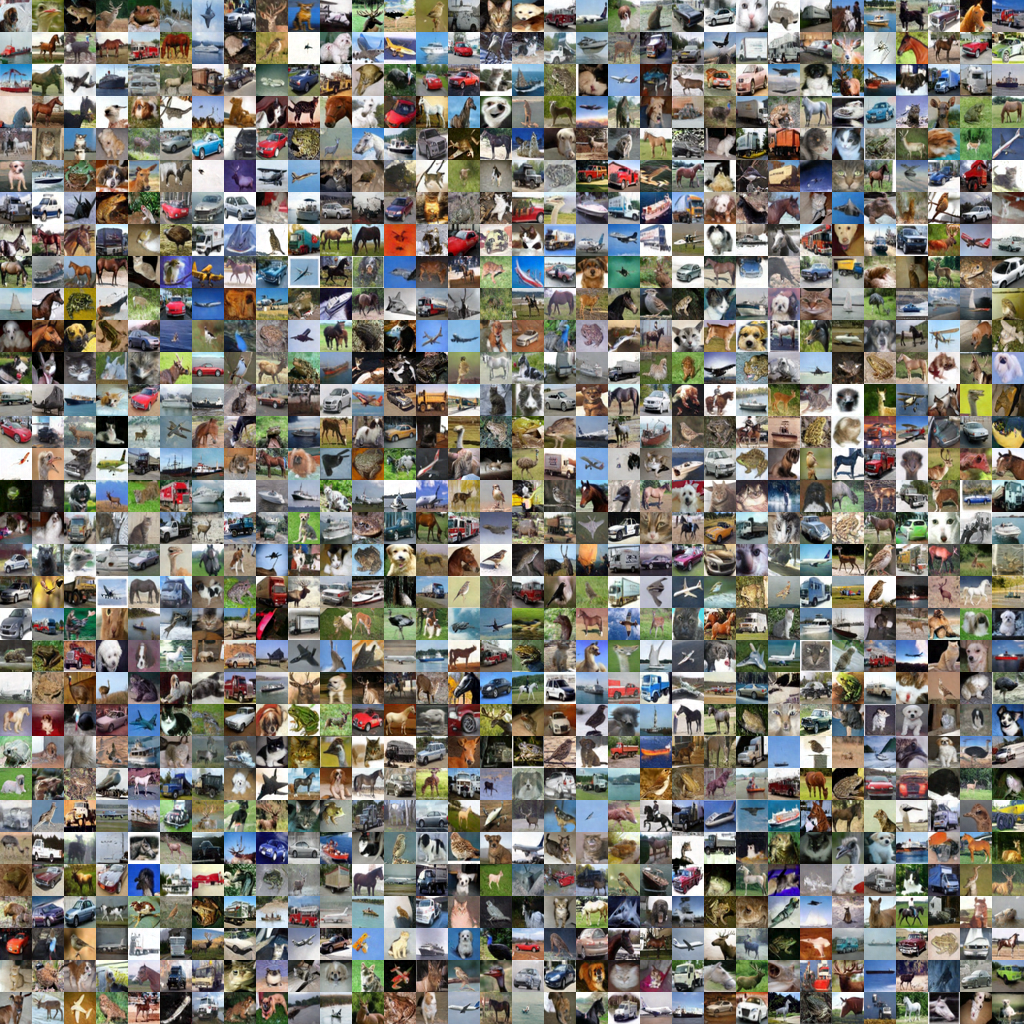}
    \caption{Uncurated samples from unconditional CIFAR-10 32$\times$32 using \method{} with single step generation (FID=3.61).}
    \label{fig:cifar_visual}
\end{figure}

%% file: sections/2_relatedworks.tex
\section{Related Works: Diffusion Distillation}
\label{sec:relatedworks}

% \subsection{Diffusion Distillation}
% \label{sec:diffusion_distillation}
Distillation methods for accelerating diffusion and flow models fall into two broad families. \emph{ODE-based} distillation exploits the teacher's Probability Flow ODE (PF-ODE) to derive regression-style objectives for a student model \citep{luhman2021knowledge,song2023consistency,liu2022rectified,salimans2022progressive,gu2023boot,meng2023distillation,yan2024perflow,frans2024one,zhu2025slimflow,geng2025mean,boffi2025flow,boffi2025build}. These approaches frame distillation as learning an ODE-consistent mapping, often enabling stable one- or few-step samplers which preserves the coupling induced by teacher models' PF-ODE.
By contrast, \emph{distribution-based} methods align the student generator’s output distribution with the teacher’s multi-step sampling distribution (or with a specified data distribution) without relying on an explicit PF-ODE. This class covers divergence- and adversarial-style matching techniques \citep{luo2023diff,yin2024one,yin2024improved,xu2024ufogen,kim2023consistency,zhou2024score,zhou2024long,zhou2024adversarial,nguyen2024swiftbrush,dao2025swiftbrush,luo2025one,zhou2025few,wang2024prolificdreamer,xie2024distillation,salimans2025multistep,zhu2025di,zhu2025soft,zheng2025ultra,wang2025uni}. 
In $f$-distil,~\cite{xu2025one} extend the VSD framework from reverse Kullback–Leibler (KL) divergence to more general $f$-divergence and use discriminator to estimate the density ratio.
A notable feature of many distribution-based methods is that they match not only the final data distribution but also the intermediate noisy-data distributions encountered during sampling; this property has also been referred to as \emph{Interpolation Distillation} \citep{liu2025blessing}.

% \cite{yang2024consistency} proposed to learn consistency flow by enforce two conditions for two consecutive timesteps.
% Given that the velocity and $x_0$ predictions can represent each other, this work can be considerred as the flow version of consistency models.
% While in Theorem 1, the authors derived another objective which can be seen as a special case of $s=0$ in our case, this objective is not unused in practice.
% Similar to previous work \cite{yan2024perflow,yoon2024sequential,heek2024multistep}, this work also applied the idea of multi-segment generation, this strategy will limit the one step generation ability of the model, as the model need at least number of segments sampling steps for generation.

% \cite{frans2024one} proposed to train a shortcut model using both Flow-Matching loss and a self-consistency loss. This method equivalently can be viewed as progressive self distillation \cite{salimans2022progressive} and support sampling on limited fixed timesteps. In comparison, our method enables more flexible sampling strategies and can get rid of the flow matching loss, leads to more efficient training.

%% file: figs/qualitative_images_SD15.tex
\begin{figure}[!h]
    \centering
    \renewcommand{\arraystretch}{1.2} % adjust spacing
    \setlength{\tabcolsep}{3pt} % spacing between columns
    \begin{tabular}{cccc}
        \textbf{Teacher} & \textbf{Student} & \textbf{Teacher} & \textbf{Student} \\ 

        % Row 1
        \includegraphics[width=0.23\textwidth]{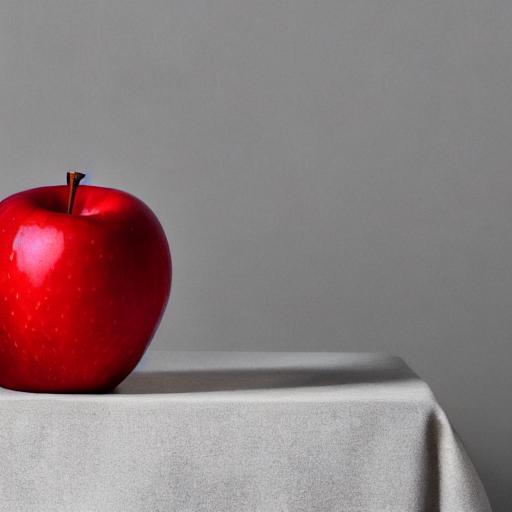} &
        \includegraphics[width=0.23\textwidth]{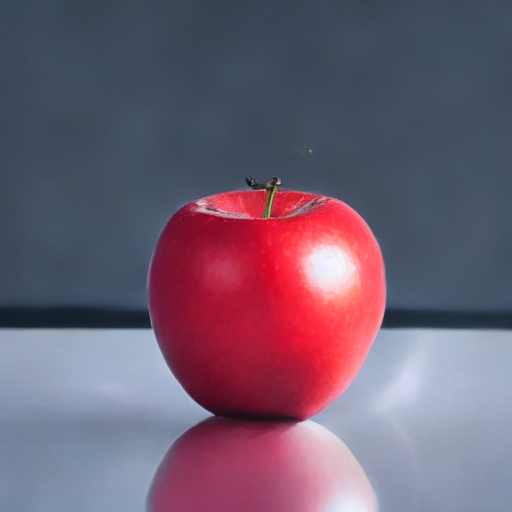} &
        \includegraphics[width=0.23\textwidth]{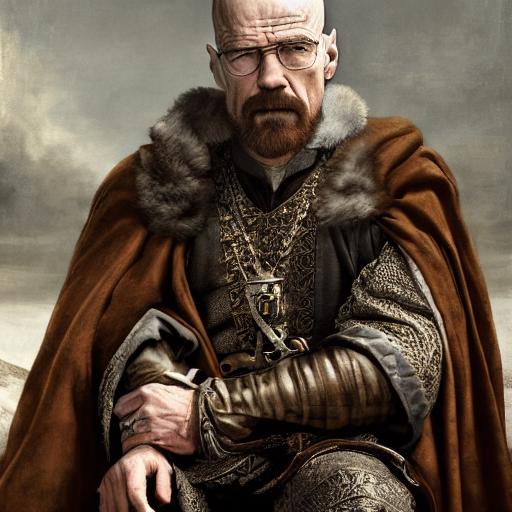} &
        \includegraphics[width=0.23\textwidth]{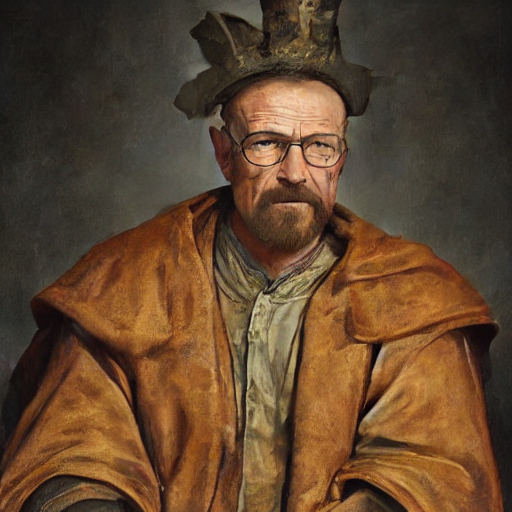} \\
        %\fbox{\rule{0pt}{0.23\textwidth}\rule{0.23\textwidth}{0pt}} &
        %\fbox{\rule{0pt}{0.23\textwidth}\rule{0.23\textwidth}{0pt}} &
        %\fbox{\rule{0pt}{0.23\textwidth}\rule{0.23\textwidth}{0pt}} &
        %\fbox{\rule{0pt}{0.23\textwidth}\rule{0.23\textwidth}{0pt}} \\
        \multicolumn{2}{c}{\small \texttt{A still life of a red apple}} &
        \multicolumn{2}{c}{\small \texttt{Walter White as a medieval king}} \\[8pt]

        % Row 2
        \includegraphics[width=0.23\textwidth]{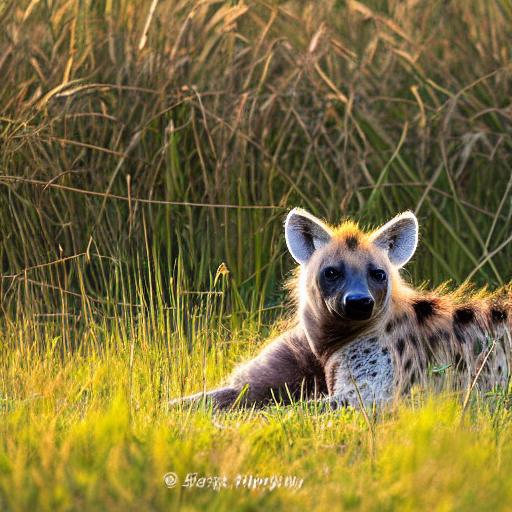} &
        \includegraphics[width=0.23\textwidth]{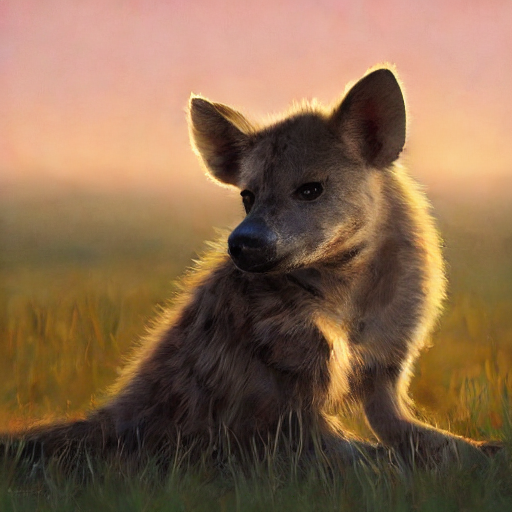} &
        \includegraphics[width=0.23\textwidth]{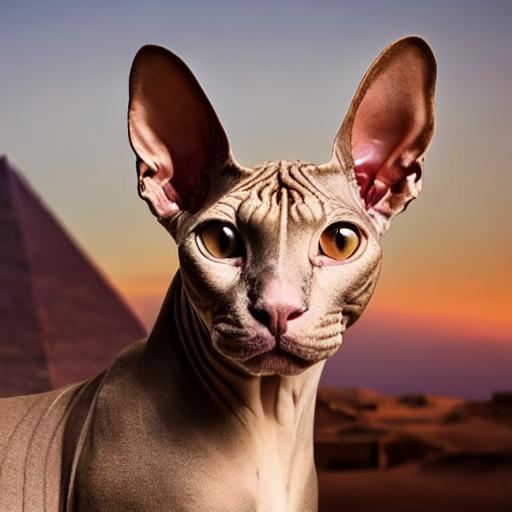} &
        \includegraphics[width=0.23\textwidth]{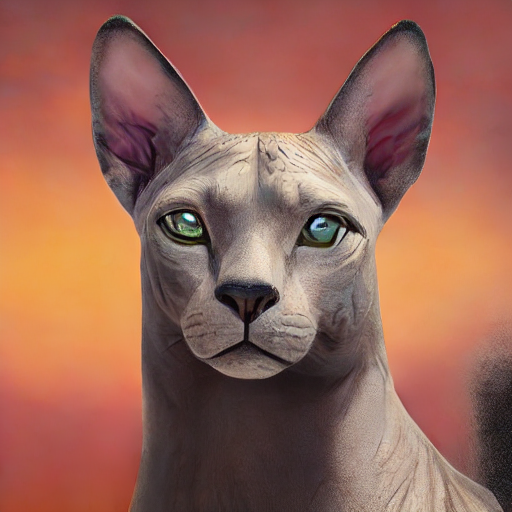} \\
        %\fbox{\rule{0pt}{0.23\textwidth}\rule{0.23\textwidth}{0pt}} &
        %\fbox{\rule{0pt}{0.23\textwidth}\rule{0.23\textwidth}{0pt}} &
        %\fbox{\rule{0pt}{0.23\textwidth}\rule{0.23\textwidth}{0pt}} &
        %\fbox{\rule{0pt}{0.23\textwidth}\rule{0.23\textwidth}{0pt}} \\
        \multicolumn{2}{c}{\small \texttt{A portrait of a hyena}} &
        \multicolumn{2}{c}{\small \texttt{A portrait of a Sphynx}} \\[8pt]

        % Row 3
        
        %\includegraphics[width=0.23\textwidth]{figs/images/sd15/160.jpg} &
        %\includegraphics[width=0.23\textwidth]{figs/images/us/160.png} &
        %\includegraphics[width=0.23\textwidth]{figs/images/sd15/179.jpg} &
        %\includegraphics[width=0.23\textwidth]{figs/images/us/179.png} \\
        %\fbox{\rule{0pt}{0.23\textwidth}\rule{0.23\textwidth}{0pt}} &
        %\fbox{\rule{0pt}{0.23\textwidth}\rule{0.23\textwidth}{0pt}} &
        %\fbox{\rule{0pt}{0.23\textwidth}\rule{0.23\textwidth}{0pt}} &
        %\fbox{\rule{0pt}{0.23\textwidth}\rule{0.23\textwidth}{0pt}} \\
        %\multicolumn{2}{c}{\small \texttt{A wounded samurai}} &
        %\multicolumn{2}{c}{\small \texttt{Iron Man}} \\[8pt]

        % Row 4
        \includegraphics[width=0.23\textwidth]{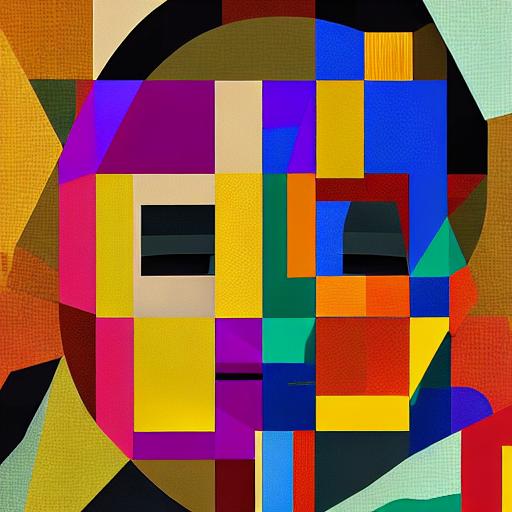} &
        \includegraphics[width=0.23\textwidth]{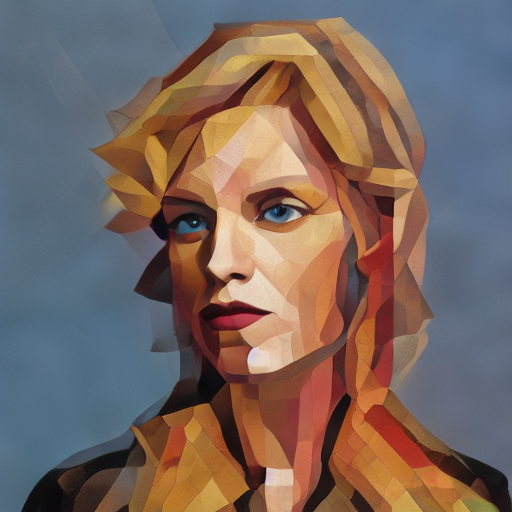} &
        \includegraphics[width=0.23\textwidth]{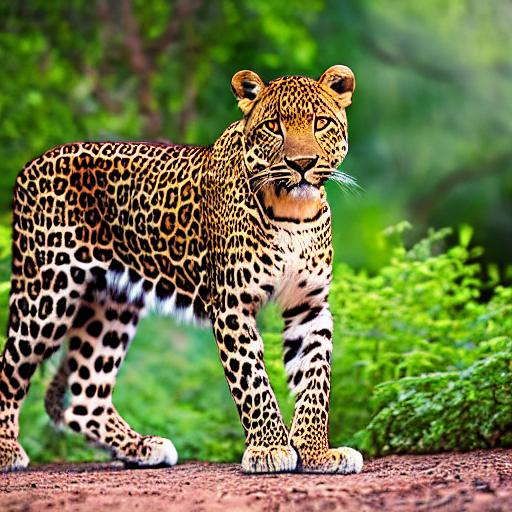} &
        \includegraphics[width=0.23\textwidth]{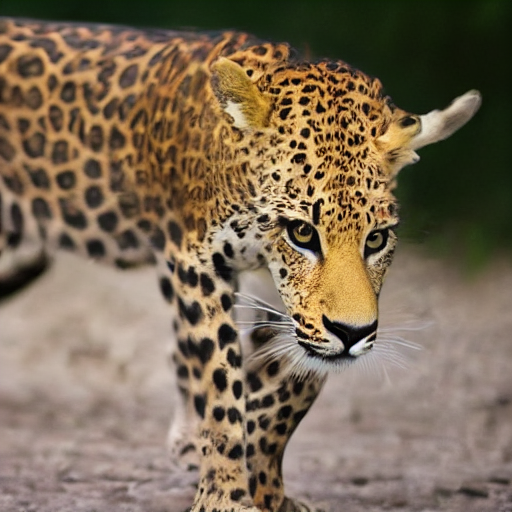} \\
        %\fbox{\rule{0pt}{0.23\textwidth}\rule{0.23\textwidth}{0pt}} &
        %\fbox{\rule{0pt}{0.23\textwidth}\rule{0.23\textwidth}{0pt}} &
        %\fbox{\rule{0pt}{0.23\textwidth}\rule{0.23\textwidth}{0pt}} &
        %\fbox{\rule{0pt}{0.23\textwidth}\rule{0.23\textwidth}{0pt}} \\
        \multicolumn{2}{c}{\small \texttt{A portrait in the style of polygonal painting}} &
        \multicolumn{2}{c}{\small \texttt{A female leopard}} \\[8pt]

    \end{tabular}
    \caption{\textbf{Qualitative comparison at $512\times512$: Teacher 50 NFEs (first, third columns) vs Student 1 NFE (second, last columns) for six prompts (left and right blocks per row).} The teacher is the Stable Diffusion v1.5~\cite{rombach2022high} model.}
    % distilled at $512\times512$ resolution. 
    \label{fig:qualitative_comparison_SD}
\end{figure}

%% file: main.bbl
\begin{thebibliography}{10}

\bibitem{banerjee2005clustering}
Arindam Banerjee, Srujana Merugu, Inderjit~S Dhillon, and Joydeep Ghosh.
\newblock Clustering with bregman divergences.
\newblock {\em Journal of machine learning research}, 6(Oct):1705--1749, 2005.

\bibitem{boffi2025build}
Nicholas~M Boffi, Michael~S Albergo, and Eric Vanden-Eijnden.
\newblock How to build a consistency model: Learning flow maps via self-distillation.
\newblock {\em arXiv preprint arXiv:2505.18825}, 2025.

\bibitem{boffi2025flow}
Nicholas~Matthew Boffi, Michael~Samuel Albergo, and Eric Vanden-Eijnden.
\newblock Flow map matching with stochastic interpolants: A mathematical framework for consistency models.
\newblock {\em Transactions on Machine Learning Research}, 2025.

\bibitem{boudier2025dipsy}
Luc Boudier, Loris Manganelli, Eleftherios Tsonis, Nicolas Dufour, and Vicky Kalogeiton.
\newblock Training-free synthetic data generation with dual ip-adapter guidance.
\newblock In {\em British Machine Vision Conference (BMVC)}, 2025.

\bibitem{bregman1967relaxation}
Lev~M Bregman.
\newblock The relaxation method of finding the common point of convex sets and its application to the solution of problems in convex programming.
\newblock {\em USSR computational mathematics and mathematical physics}, 7(3):200--217, 1967.

\bibitem{courant2025pulp}
Robin Courant, Xi~Wang, David Loiseaux, Marc Christie, and Vicky Kalogeiton.
\newblock Pulp motion: Framing-aware multimodal camera and human motion generation.
\newblock {\em arXiv preprint arXiv:2510.05097}, 2025.

\bibitem{dao2025swiftbrush}
Trung Dao, Thuan~Hoang Nguyen, Thanh Le, Duc Vu, Khoi Nguyen, Cuong Pham, and Anh Tran.
\newblock Swiftbrush v2: Make your one-step diffusion model better than its teacher.
\newblock In {\em European Conference on Computer Vision}, pages 176--192. Springer, 2025.

\bibitem{degeorge2025farimagenettexttoimagegeneration}
Lucas Degeorge, Arijit Ghosh, Nicolas Dufour, David Picard, and Vicky Kalogeiton.
\newblock How far can we go with imagenet for text-to-image generation?
\newblock {\em arXiv}, 2025.

\bibitem{dufour2024don}
Nicolas Dufour, Victor Besnier, Vicky Kalogeiton, and David Picard.
\newblock Don't drop your samples! coherence-aware training benefits conditional diffusion.
\newblock In {\em Proceedings of the IEEE/CVF Conference on Computer Vision and Pattern Recognition}, pages 6264--6273, 2024.

\bibitem{frans2024one}
Kevin Frans, Danijar Hafner, Sergey Levine, and Pieter Abbeel.
\newblock One step diffusion via shortcut models.
\newblock {\em arXiv preprint arXiv:2410.12557}, 2024.

\bibitem{geng2025mean}
Zhengyang Geng, Mingyang Deng, Xingjian Bai, J~Zico Kolter, and Kaiming He.
\newblock Mean flows for one-step generative modeling.
\newblock {\em arXiv preprint arXiv:2505.13447}, 2025.

\bibitem{gu2023boot}
Jiatao Gu, Shuangfei Zhai, Yizhe Zhang, Lingjie Liu, and Joshua~M Susskind.
\newblock Boot: Data-free distillation of denoising diffusion models with bootstrapping.
\newblock In {\em ICML 2023 Workshop on Structured Probabilistic Inference $\{$$\backslash$\&$\}$ Generative Modeling}, 2023.

\bibitem{heusel2017gans}
Martin Heusel, Hubert Ramsauer, Thomas Unterthiner, Bernhard Nessler, and Sepp Hochreiter.
\newblock Gans trained by a two time-scale update rule converge to a local nash equilibrium.
\newblock {\em Advances in neural information processing systems}, 30, 2017.

\bibitem{karras2022elucidating}
Tero Karras, Miika Aittala, Timo Aila, and Samuli Laine.
\newblock Elucidating the design space of diffusion-based generative models.
\newblock {\em Advances in neural information processing systems}, 35:26565--26577, 2022.

\bibitem{kim2023consistency}
Dongjun Kim, Chieh-Hsin Lai, Wei-Hsiang Liao, Naoki Murata, Yuhta Takida, Toshimitsu Uesaka, Yutong He, Yuki Mitsufuji, and Stefano Ermon.
\newblock Consistency trajectory models: Learning probability flow ode trajectory of diffusion.
\newblock {\em arXiv preprint arXiv:2310.02279}, 2023.

\bibitem{kim2025preference}
Yeongmin Kim, Heesun Bae, Byeonghu Na, and Il-Chul Moon.
\newblock Preference optimization by estimating the ratio of the data distribution.
\newblock {\em arXiv preprint arXiv:2505.19601}, 2025.

\bibitem{krizhevsky2009learning}
Alex Krizhevsky, Geoffrey Hinton, et~al.
\newblock Learning multiple layers of features from tiny images.
\newblock 2009.

\bibitem{lin2014microsoft}
Tsung-Yi Lin, Michael Maire, Serge Belongie, James Hays, Pietro Perona, Deva Ramanan, Piotr Doll{\'a}r, and C~Lawrence Zitnick.
\newblock Microsoft coco: Common objects in context.
\newblock In {\em Computer vision--ECCV 2014: 13th European conference, zurich, Switzerland, September 6-12, 2014, proceedings, part v 13}, pages 740--755. Springer, 2014.

\bibitem{liu2022rectified}
Qiang Liu.
\newblock Rectified flow: A marginal preserving approach to optimal transport.
\newblock {\em arXiv preprint arXiv:2209.14577}, 2022.

\bibitem{liu2025blessing}
Qiang Liu.
\newblock Icml tutorial on the blessing of flow.
\newblock {\em International conference on machine learning}, 2025.

\bibitem{lu2024mace}
Shilin Lu, Zilan Wang, Leyang Li, Yanzhu Liu, and Adams Wai-Kin Kong.
\newblock Mace: Mass concept erasure in diffusion models.
\newblock In {\em Proceedings of the IEEE/CVF Conference on Computer Vision and Pattern Recognition}, pages 6430--6440, 2024.

\bibitem{luhman2021knowledge}
Eric Luhman and Troy Luhman.
\newblock Knowledge distillation in iterative generative models for improved sampling speed.
\newblock {\em arXiv preprint arXiv:2101.02388}, 2021.

\bibitem{luo2023diff}
Weijian Luo, Tianyang Hu, Shifeng Zhang, Jiacheng Sun, Zhenguo Li, and Zhihua Zhang.
\newblock Diff-instruct: A universal approach for transferring knowledge from pre-trained diffusion models.
\newblock {\em Advances in Neural Information Processing Systems}, 36:76525--76546, 2023.

\bibitem{luo2024one}
Weijian Luo, Zemin Huang, Zhengyang Geng, J~Zico Kolter, and Guo-jun Qi.
\newblock One-step diffusion distillation through score implicit matching.
\newblock {\em Advances in Neural Information Processing Systems}, 37:115377--115408, 2024.

\bibitem{luo2025one}
Weijian Luo, Zemin Huang, Zhengyang Geng, J~Zico Kolter, and Guo-jun Qi.
\newblock One-step diffusion distillation through score implicit matching.
\newblock {\em Advances in Neural Information Processing Systems}, 37:115377--115408, 2025.

\bibitem{meng2023distillation}
Chenlin Meng, Robin Rombach, Ruiqi Gao, Diederik Kingma, Stefano Ermon, Jonathan Ho, and Tim Salimans.
\newblock On distillation of guided diffusion models.
\newblock In {\em Proceedings of the IEEE/CVF Conference on Computer Vision and Pattern Recognition}, pages 14297--14306, 2023.

\bibitem{nguyen2024swiftbrush}
Thuan~Hoang Nguyen and Anh Tran.
\newblock Swiftbrush: One-step text-to-image diffusion model with variational score distillation.
\newblock In {\em Proceedings of the IEEE/CVF Conference on Computer Vision and Pattern Recognition}, pages 7807--7816, 2024.

\bibitem{nielsen2009sided}
Frank Nielsen and Richard Nock.
\newblock Sided and symmetrized bregman centroids.
\newblock {\em IEEE transactions on Information Theory}, 55(6):2882--2904, 2009.

\bibitem{parmar2022aliased}
Gaurav Parmar, Richard Zhang, and Jun-Yan Zhu.
\newblock On aliased resizing and surprising subtleties in gan evaluation.
\newblock In {\em Proceedings of the IEEE/CVF Conference on Computer Vision and Pattern Recognition}, pages 11410--11420, 2022.

\bibitem{peebles2023scalable}
William Peebles and Saining Xie.
\newblock Scalable diffusion models with transformers.
\newblock In {\em Proceedings of the IEEE/CVF international conference on computer vision}, pages 4195--4205, 2023.

\bibitem{poole2022dreamfusion}
Ben Poole, Ajay Jain, Jonathan~T Barron, and Ben Mildenhall.
\newblock Dreamfusion: Text-to-3d using 2d diffusion.
\newblock {\em arXiv preprint arXiv:2209.14988}, 2022.

\bibitem{rombach2022high}
Robin Rombach, Andreas Blattmann, Dominik Lorenz, Patrick Esser, and Bj{\"o}rn Ommer.
\newblock High-resolution image synthesis with latent diffusion models.
\newblock In {\em Proceedings of the IEEE/CVF Conference on Computer Vision and Pattern Recognition}, pages 10684--10695, 2022.

\bibitem{salimans2016improved}
Tim Salimans, Ian Goodfellow, Wojciech Zaremba, Vicki Cheung, Alec Radford, and Xi~Chen.
\newblock Improved techniques for training gans.
\newblock {\em Advances in neural information processing systems}, 29, 2016.

\bibitem{salimans2022progressive}
Tim Salimans and Jonathan Ho.
\newblock Progressive distillation for fast sampling of diffusion models.
\newblock {\em arXiv preprint arXiv:2202.00512}, 2022.

\bibitem{salimans2025multistep}
Tim Salimans, Thomas Mensink, Jonathan Heek, and Emiel Hoogeboom.
\newblock Multistep distillation of diffusion models via moment matching.
\newblock {\em Advances in Neural Information Processing Systems}, 37:36046--36070, 2025.

\bibitem{schuhmann2022laion}
Christoph Schuhmann, Romain Beaumont, Richard Vencu, Cade Gordon, Ross Wightman, Mehdi Cherti, Theo Coombes, Aarush Katta, Clayton Mullis, Mitchell Wortsman, et~al.
\newblock Laion-5b: An open large-scale dataset for training next generation image-text models.
\newblock {\em neurips}, 2022.

\bibitem{song2023consistency}
Yang Song, Prafulla Dhariwal, Mark Chen, and Ilya Sutskever.
\newblock Consistency models.
\newblock 2023.

\bibitem{song2020score}
Yang Song, Jascha Sohl-Dickstein, Diederik~P Kingma, Abhishek Kumar, Stefano Ermon, and Ben Poole.
\newblock Score-based generative modeling through stochastic differential equations.
\newblock In {\em International Conference on Learning Representations}, 2020.

\bibitem{sugiyama2008direct}
Masashi Sugiyama, Shinichi Nakajima, Hisashi Kashima, Paul von Bünau, and Motoaki Kawanabe.
\newblock Direct importance estimation with model selection and its application to covariate shift adaptation.
\newblock {\em Advances in Neural Information Processing Systems}, 20, 2008.

\bibitem{sugiyama2012density}
Masashi Sugiyama, Taiji Suzuki, and Takafumi Kanamori.
\newblock Density-ratio matching under the bregman divergence: a unified framework of density-ratio estimation.
\newblock {\em Annals of the Institute of Statistical Mathematics}, 64(5):1009--1044, 2012.

\bibitem{wang2025akira}
Xi~Wang, Robin Courant, Marc Christie, and Vicky Kalogeiton.
\newblock Akira: Augmentation kit on rays for optical video generation.
\newblock In {\em Proceedings of the Computer Vision and Pattern Recognition Conference}, pages 2609--2619, 2025.

\bibitem{wang2025uni}
Yifei Wang, Weimin Bai, Colin Zhang, Debing Zhang, Weijian Luo, and He~Sun.
\newblock Uni-instruct: One-step diffusion model through unified diffusion divergence instruction.
\newblock {\em arXiv preprint arXiv:2505.20755}, 2025.

\bibitem{wang2024prolificdreamer}
Zhengyi Wang, Cheng Lu, Yikai Wang, Fan Bao, Chongxuan Li, Hang Su, and Jun Zhu.
\newblock Prolificdreamer: High-fidelity and diverse text-to-3d generation with variational score distillation.
\newblock {\em Advances in Neural Information Processing Systems}, 36, 2024.

\bibitem{xie2024distillation}
Sirui Xie, Zhisheng Xiao, Diederik~P Kingma, Tingbo Hou, Ying~Nian Wu, Kevin~Patrick Murphy, Tim Salimans, Ben Poole, and Ruiqi Gao.
\newblock Em distillation for one-step diffusion models.
\newblock {\em arXiv preprint arXiv:2405.16852}, 2024.

\bibitem{xu2024ufogen}
Yanwu Xu, Yang Zhao, Zhisheng Xiao, and Tingbo Hou.
\newblock Ufogen: You forward once large scale text-to-image generation via diffusion gans.
\newblock In {\em Proceedings of the IEEE/CVF Conference on Computer Vision and Pattern Recognition}, pages 8196--8206, 2024.

\bibitem{xu2025one}
Yilun Xu, Weili Nie, and Arash Vahdat.
\newblock One-step diffusion models with $ f $-divergence distribution matching.
\newblock {\em arXiv preprint arXiv:2502.15681}, 2025.

\bibitem{yan2024perflow}
Hanshu Yan, Xingchao Liu, Jiachun Pan, Jun~Hao Liew, Qiang Liu, and Jiashi Feng.
\newblock Perflow: Piecewise rectified flow as universal plug-and-play accelerator.
\newblock {\em arXiv preprint arXiv:2405.07510}, 2024.

\bibitem{yin2024improved}
Tianwei Yin, Micha{\"e}l Gharbi, Taesung Park, Richard Zhang, Eli Shechtman, Fredo Durand, and Bill Freeman.
\newblock Improved distribution matching distillation for fast image synthesis.
\newblock {\em Advances in neural information processing systems}, 37:47455--47487, 2024.

\bibitem{yin2024one}
Tianwei Yin, Micha{\"e}l Gharbi, Richard Zhang, Eli Shechtman, Fredo Durand, William~T Freeman, and Taesung Park.
\newblock One-step diffusion with distribution matching distillation.
\newblock In {\em Proceedings of the IEEE/CVF Conference on Computer Vision and Pattern Recognition}, pages 6613--6623, 2024.

\bibitem{zheng2025ultra}
Haoyang Zheng, Xinyang Liu, Cindy~Xiangrui Kong, Nan Jiang, Zheyuan Hu, Weijian Luo, Wei Deng, and Guang Lin.
\newblock Ultra-fast language generation via discrete diffusion divergence instruct.
\newblock {\em arXiv preprint arXiv:2509.25035}, 2025.

\bibitem{zhou2025few}
Mingyuan Zhou, Yi~Gu, and Zhendong Wang.
\newblock Few-step diffusion via score identity distillation.
\newblock {\em arXiv preprint arXiv:2505.12674}, 2025.

\bibitem{zhou2024long}
Mingyuan Zhou, Zhendong Wang, Huangjie Zheng, and Hai Huang.
\newblock Long and short guidance in score identity distillation for one-step text-to-image generation.
\newblock {\em arXiv preprint arXiv:2406.01561}, 2024.

\bibitem{zhou2024adversarial}
Mingyuan Zhou, Huangjie Zheng, Yi~Gu, Zhendong Wang, and Hai Huang.
\newblock Adversarial score identity distillation: Rapidly surpassing the teacher in one step.
\newblock {\em arXiv preprint arXiv:2410.14919}, 2024.

\bibitem{zhou2024score}
Mingyuan Zhou, Huangjie Zheng, Zhendong Wang, Mingzhang Yin, and Hai Huang.
\newblock Score identity distillation: Exponentially fast distillation of pretrained diffusion models for one-step generation.
\newblock In {\em Forty-first International Conference on Machine Learning}, 2024.

\bibitem{zhu2025slimflow}
Yuanzhi Zhu, Xingchao Liu, and Qiang Liu.
\newblock Slimflow: Training smaller one-step diffusion models with rectified flow.
\newblock In {\em European Conference on Computer Vision}, pages 342--359. Springer, 2025.

\bibitem{zhu2025di}
Yuanzhi Zhu, Xi~Wang, St{\'e}phane Lathuili{\`e}re, and Vicky Kalogeiton.
\newblock Di$\mathtt{[M]}$o: Distilling masked diffusion models into one-step generator.
\newblock {\em arXiv preprint arXiv:2503.15457}, 2025.

\bibitem{zhu2025soft}
Yuanzhi Zhu, Xi~Wang, St{\'e}phane Lathuili{\`e}re, and Vicky Kalogeiton.
\newblock Soft-di [m] o: Improving one-step discrete image generation with soft embeddings.
\newblock {\em arXiv preprint arXiv:2509.22925}, 2025.

\end{thebibliography}
